\documentclass[12pt]{article}
\usepackage[utf8]{inputenc}
\usepackage{algorithm}
\usepackage{algpseudocode}
\usepackage{amsmath}
\usepackage{amssymb}
\usepackage{amsthm}
\usepackage{xargs}
\usepackage{bm}
\usepackage{graphicx}
\usepackage{booktabs}
\usepackage{tabularx}
\usepackage[colorlinks=true, linkcolor=blue]{hyperref}
\usepackage{scalerel,stackengine}
\usepackage{xcolor}
\usepackage{multirow}
\usepackage{makecell}
\usepackage{chngcntr}
\usepackage{apptools}
\usepackage{dsfont}
\usepackage[T1]{fontenc}
\usepackage[numbers]{natbib}
\usepackage{longtable}
\usepackage{thmtools, thm-restate}
\usepackage{cancel}
\usepackage{tikz}

\DeclareMathOperator*{\argmax}{arg\,max}
\DeclareMathOperator*{\argmin}{arg\,min}

\theoremstyle{plain}

\newtheorem{lemma}{Lemma}[]
\newtheorem{corollary}{Corollary}[]
\newcommandx{\pto}[0]{\overset{P}{\to}}

\newcommand{\blind}{1}

\usepackage{geometry}
\begin{document}

\def\spacingset#1{\renewcommand{\baselinestretch}%
{#1}\small\normalsize} \spacingset{1}

\if1\blind
{
  \title{\bf Lassoed Tree Boosting}
  \author{Alejandro Schuler
    \hspace{.2cm}\\
    Division of Biostatistics, University of California, Berkeley\\
    \\
    Yi Li \\
    Division of Biostatistics, University of California, Berkeley \\
    \\
    Mark van der Laan \\
    Division of Biostatistics, University of California, Berkeley}
  \maketitle
} \fi

\if0\blind
{
  \bigskip
  \bigskip
  \bigskip
  \begin{center}
    {\LARGE\bf Lassoed Tree Boosting}
\end{center}
  \medskip
} \fi

\bigskip
\begin{abstract}
 
Gradient boosting performs exceptionally in most prediction problems and scales well to large datasets. 
In this paper we prove that a ``lassoed'' gradient boosted tree algorithm with early stopping achieves faster than $n^{-1/4}$ L2 convergence in the large nonparametric space of cadlag functions of bounded sectional variation. 
This rate is remarkable because it does not depend on the dimension, sparsity, or smoothness.
We use simulation and real data to confirm our theory and demonstrate empirical performance and scalability on par with standard boosting.
Our convergence proofs are based on a novel, general theorem on early stopping with empirical loss minimizers of nested Donsker classes. 
\end{abstract}

\noindent%
\vfill

\newpage
% \spacingset{1.45} % DON'T change the spacing!

\section{Introduction}

In regression our task is to find a function that maps features to an outcome such that the expected loss is minimized \cite{Hastie2009-ix}. In the past decades a huge number of flexible regression methods have been developed that effectively search over high- or infinite-dimensional function spaces. These are often collectively called ``machine learning'' methods for regression. 
Gradient-boosted trees in particular have proven consistently successful in real-world prediction settings \cite{Friedman2001-jo, Bojer2021-np, xgboost} and, unlike deep learning approaches, require minimal parameter tuning or expert knowledge.

L2 convergence is a well-studied property of regression algorithms. L2 convergence rate measures how quickly generalization MSE decreases as the size of the training sample increases. Fast convergence rates are desirable in predictive settings because they (asymptotically) guarantee more efficient use of limited data. 

In many causal inference settings fast rates are in fact \textit{required} to build valid confidence intervals. For example, when estimating the average treatment effect from observational data, a necessary condition for the asymptotic normality of the TMLE and AIPW estimators is that the propensity score and outcome regression models converge to their respective truths in root-mean-square generalization error at a rate of $o_P(n^{-1/4})$ \cite{Tsiatis2007-py, Van_der_Laan2003-du}. Thus the development of fast-converging, nonparametric regression methods is critical for efficient statistical inference.

There are many rate results for gradient boosting in the literature but there are gaps between theory and practice. Most results apply only to abstracted or special-case versions of the boosting algorithm in question that do not reflect popular implementations (they tend to ignore the greediness of regression trees; i.e. CART does not generate a globally optimal tree) \cite{buhlmann2003boosting, blanchard2003rate, bickel2006some, buhlmann2002consistency, mannor2003greedy, zhang2005boosting}. Others impose limiting assumptions on the true regression functions (e.g. smoothness, sparsity) \cite{buhlmann2003boosting, luo2016high, mannor2002consistency}.

Some assumptions about the true regression function are always required becasue the convergence rate for a regression estimator depends on the function class that is being searched. It is difficult to prove rates faster than $o_P(n^{-1/4})$ without assuming that the regression functions are differentiable some number of times or only depend on a very low-dimensional subset of predictors. In fact, the best-possible rates in these settings are very slow without heroic assumptions \cite{Stone1982-up}. This is often referred to as the \textit{curse of dimensionality} \cite{Robins2008-fx, Fang2019-or}. In 2015, however, van der Laan \cite{hal-og} showed that dimension-free $o_P(n^{-1/4})$ rates (or better) are attainable under more tenable assumptions with a regression method called the \textit{Highly Adaptive Lasso} (HAL) \cite{hal, Fang2019-or, Bibaut2019-zf}. Instead of assuming smoothness or sparsity, HAL assumes that the true function is cadlag of bounded sectional variation. Roughly speaking, this restricts the amount the true function can go up and down over its domain in a total, global sense, instead of restricting its behavior locally. 

Unfortunately, asymptotic behavior is meaningless if in practice the method cannot handle large, high-dimensional data or if empirical performance is poor. Although HAL is theoretically powerful, its memory use and computational cost scale exponentially with covariate dimension \cite{hal, Fang2019-or} and ``scalable'' modifications to the algorithm have not been theoretically proven to preserve the rate \cite{hal9001}.  

Here we present an algorithm (lassoed boosting) that is scalable \textit{and} guarantees fast convergence rates in a meaningfully large nonparametric space. Our proposed method, which we call ``lassoed tree boosting''\footnote{In a previous version of this paper we referred to this algorithm as the ``selectively adaptive lasso'' (SAL) to emphasize the connection to HAL. In this version we have chosen instead to emphasize the more concrete connections to gradient boosting.} (LTB; or just ``lassoed boosting''), is an ensemble of boosted regression trees with the tree weights set by lasso regression and tuned with early-stopping.
As a consequence, it is trivial to implement with off-the-shelf software and scalable to large datasets. We show that this exact algorithm, without modification, retains HAL's convergence rate in the same nonparametric setting. To accomplish this, we prove some general theoretical results pertaining to empirical loss minimization in nested Donsker classes. 

\section{Notation and Preliminaries}

Throughout the paper we adopt the empirical process notation $Pf = \int f(Z) dP$ and $P_n f = \frac{1}{n}\sum_i f(Z_i)$. In this notation these operators do not average over any potential randomness in $f$, so, for example, $Pf_n$ is a random variable if $f_n$ is a random function learned from data. We reserve $\tilde P_n f_n$ for the case where $f_n$ is random, but independent from the random measure $\tilde P_n$ (e.g. when $\tilde P_n$ is the empirical measure of a validation set).  We use $\|f\|$ to indicate an $L2$ norm $\sqrt{Pf^2}$ unless otherwise noted. To indicate the probability measure of a set we use braces, e.g. $P\{Z < c\}$.

Let $X_i,Y_i \in \mathcal X \times \mathbb R$ be IID across $i$ and with a generic $X,Y$ that have some joint distribution $P$. 
We will take $\mathcal X = [0,1]^p$ without loss of generality for applications with bounded covariates. 
We use $(\mathbf X, \mathbf Y)$ to denote the training data of shape $(n \times p)$ and $n$ and $(\tilde{\mathbf X}, \tilde{\mathbf Y})$ to denote validation data of shape $(\tilde n \times p)$ and $\tilde n$.
When we write e.g. $f(\mathbf X)$ we mean the row-wise application of the function to each $X_i$ with the results concatenated into a column vector. 
If $f(x) = [f_1(x) \dots f_K(x)]$ is vector-valued, we use $f(\mathbf X)$ to denote the matrix with columns $[f_1(\mathbf X) \dots f_K(\mathbf X)]$.

Let $L$ be some loss (e.g. mean-squared error), which we construct such as to take a prediction function $f$ as an input and return a function of $X,Y$ as output. For example, if we want $L$ to be squared error loss, we let $L(f)(X,Y) = (f(X) - Y)^2$. Throughout we abbreviate $Lf = L(f)$.
Let $f = \argmin_{f:\mathcal X \to \mathbb R} PLf$.
This is the standard regression setup where our goal is to estimate the function $f$ from $n$ samples of the vector of predictors $X$ and the outcome $Y$. 

In our theoretical results we give convergence rates of some estimator $f_n$ in terms of the \textit{divergence} (or ``loss-based dissimilarity'') $P(Lf_n -Lf)$. However, we are usually interested in the \textit{L2 norm} $\|f_n - f\|_2 = \sqrt{P(f_n - f)^2}$ of regressions in order to prove efficiency in estimating some target parameter. Convergence in divergence typically implies L2 convergence at the same rate under very mild conditions on the loss function\footnote{When the loss is mean-squared-error, we see directly that $P(L f_n - Lf) = P(f_n - f)^2$ via iterated expectation (conditioning on $X$) and recognizing $f = E[Y|X]$ in this case. 
} which are satisfied in all standard applications \cite{Bibaut2019-zf}.

\paragraph*{Cadlag and Sectional Variation}
Throughout the paper we focus on the class of cadlag functions of bounded sectional variation. ``Cadlag'' (\textit{continu à droite, limites à gauche}) means right-continuous with left limits \cite{neuhaus1971weak}. So, for example, all continuous functions are cadlag. Cadlag functions correspond 1-to-1 with signed measures the same way that cumulative distribution functions correspond with probability measures. We use $\mathcal F$ to denote the set of cadlag functions.

The sectional variation norm\footnote{The \textit{sectional} variation norm is also called ``Hardy-Krause variation''. Sectional variation rectifies a commonly known failure of the conceptually simpler ``Vitali variation'', which is that functions that don't vary in one or more argument are assigned zero variation. Both these notions of variation are distinct from the more common notion of \textit{total} variation  \cite{Fang2019-or} which corresponds to the sum (integral) of the perimeters (Hausdorff measures) of all level sets of a function. All three are equivalent for functions of a single variable, and all capture some notion of how much a function ``goes up and down'' over its whole domain.} of a cadlag function $f$ on $[0,1]^p$ is given by
$\|f\|_v =  \sum_{s \subseteq \{1\dots p\}}\int_{0_s}^{1_s} |df(x_s)|$ where the $x_s$ notation means that all elements of $x$ not in the set of coordinates $s$ are fixed at 0. We use the notation $\|\cdot \|_v$ to distinguish sectional variation norm from a standard L2 norm. Intuitively, a variation norm is a way of measuring how much a function goes up and down in total over its entire domain. For a function of one variable, the variation norm is the total ``elevation gain'' plus ``elevation loss'' over the domain, up to a constant. Given some constant $M$, we use $\mathcal F(M)$ to denote all cadlag functions with $\| f\|_v \le M$.

$\mathcal F(M)$ is an extremely rich class of functions: $f \in \mathcal F(M)$ may be highly nonlinear and even discontinuous. 
Functions that are not in $\mathcal F(M)$ are often pathological, e.g. $f(x) = \cos(1/x)$. In cases where they are not (e.g. when $f(x)$ is the indicator of the unit ball) they can always be approximated arbitrarily well by a function in $\mathcal F(M)$ by increasing the bound $M$ (indeed $\cup_M \mathcal F(M)$ is dense in $L2$).
Throughout, we presume that the truth $f$ is a cadlag function of sectional variation norm. We omit notating the bound $M$ where it is clear from context or not important to the immediate argument.

\paragraph*{Boosted Trees} We presume some familiarity with gradient boosted trees (GBT) but provide a brief review here. A GBT fit $f_K$ is an additive ensemble of $K$ regression trees $h_k$: $f_K(x) = \sum_k^K \epsilon h_k(x)$. The value $\epsilon$ is typically set to some small number and is called the \textit{learning rate}. Each tree $h_k$ in the sequence is found by regressing the current residual $\mathbf Y-f_{K-1}(\mathbf X)$ (for losses besides squared error the residual is different) onto the predictors using a greedy recursive partitioning algorithm. The total number of trees $K$ is often set by \textit{early stopping validation}: at each iteration the validation loss $\tilde P_n L f_K$ is computed and fitting stops and returns $f_K$ if $\tilde P_n L f_K < \tilde P_n L f_{K+1}$, i.e. if validation loss increased after adding the current tree.

% A single regression tree is itself an additive ensemble of indicator functions for axis-aligned rectangles such that the indicated regions are a disjoint cover of the domain. For a tree with $B$ leaves we will write $t(x) = \sum_b^B \alpha_b h_b(x)$ where $h(x) = 1(c_1 \le x < c_2)$ indicate rectangles in $[0,1]^p$ (the inequalities are taken to hold in all dimensions). A GBT fit can therefore be written as a linear combination of many leaves: $f(x) = \sum_j^J  \sum_b^B \epsilon \alpha_{k,b} h_{k,b}(x) = \sum_k^K \beta_k h_k(x) = H(x)\beta$ for $H(x) = [h_1(x) \dots h_K(x)]$.

\section{Lassoed Tree Boosting}
\label{sec:algorithm}

Lassoed tree boosting is an ensemble of boosted regression trees with the tree weights set by lasso regression and tuned with early-stopping. In more detail, the steps in the algorithm are as follows:
\begin{enumerate}
    \item Fit a gradient boosted tree (GBT) ensemble $f_K(x) = \sum_k^K \epsilon h_k(x)$ to the training data. Hyperparameters may be tuned in any way, but we recommend using early stopping based on validation error to determine the number of trees $K$ and an outer early stopping loop to determine the maximum depth of each tree $D$ at a minimum.
    \item Decompose the fit into a sum of trees $f = \sum_k \beta_k h_k(x)$ and construct design matrices $H_k(\mathbf X) = [h_1 \dots h_k](\mathbf X)$ for the training and validation set where each column contains the predictions from each tree.
    \item Regress the training outcome $Y$ onto the training design matrix $H = H_k(\mathbf X)$ using lasso (i.e. solving $\beta = \argmin_\beta P_n L H\beta + \lambda \|\beta\|_1$). Select $\lambda$ by minimizing validation error over the lasso path. Record the L1 norm $\|\beta\|_1$ and validation loss $\tilde P_n L H \beta$ of all solutions $\beta(\lambda)$ over the regularization path.
    \item Add a number of trees to the GBT ensemble and repeat steps 2 and 3. Stop and return the previous lasso fit if the validation error of the current optimal solution is greater than the validation error of any solution  with smaller L1 norm from a previous iteration. Otherwise repeat this step. The final fit is given by $f(x) = H(x)\beta$.
\end{enumerate}

% Example python code implementing lassoed boosting can be found at \url{https://tinyurl.com/4te5wcb3}.

This algorithm is practical, fast, and scalable. In terms of computation it comes down to fitting a standard boosting model and then solving a small number of moderately large lasso problems. Usually the algorithm will stop after a small number of lasso iterations because we have already done early stopping for the unlassoed boosting model. 

Many GBT packages allow the user to impose an L1 regularization rate on the leaf values of the trees, but this is always approximated in a stagewise fashion \cite{xgboost}.  Similar regularization schemes have been considered in the boosting literature before \cite{mannor2002consistency}. However, for our theory to hold the lasso solution must be recomputed exactly over all trees jointly. A naive implementation would re-optimize all the tree weights after the addition of every weak learner which would add huge computational burden. Our practical innovation is to tie this regularization in with early stopping so that it only needs to be done a small number of times.

Our theoretical innovation is to show that this algorithm is guaranteed by the theory presented in the subsequent sections to attain an $O_P(n^{-1/3}(\log n)^{2(p-1)/3})$ L2 convergence rate in the large, nonparametric space of cadlag functions of bounded sectional variation. This rate is dimension-free up to the log factor. 
Many variations of this algorithm (e.g. cross-validated early stopping) are possible and are supported by our general theory. However, our goal here is to present a canonical algorithm, support it with rigorous theory, and demonstrate practical utility. Further experimentation is certainly warranted in the future.

\section{Convergence Rate}

In this section we show that the lassoed boosting algorithm attains a $O_P(n^{-1/3}(\log n)^{2(p-1)/3})$ L2 convergence rate in the space of cadlag functions of bounded sectional variation. 

We first analyze a simplified version of the lassoed boosting algorithm where a) the boosting ``bases'' are simple step functions $1(c \le x)$ instead of trees and b) the L1 norm bound is set a-priori. In this setting we can finitely enumerate all the possible bases that could be generated by the algorithm for given training data. In the limit where they are all included, we show our simplified lassoed boosting algorithm is identical to HAL (the highly adaptive lasso), which is known to achieve the $O_P(n^{-1/3}(\log n)^{2(p-1)/3})$ rate. We then prove that we can replace the simpler step functions first with rectangular indicators and then with regression trees without sacrificing the rate. 

We then present a novel, general result that shows that early-stopping validation of our lassoed boosting algorithm does not affect that rate: in other words, if we add trees only until validation error increases the asymptotic behavior is the same as if we had included all of them. This result holds uniformly over the order in which bases are added which lets us make practical choices that improve the finite-sample performance of lassoed boosting.

We conclude with some arguments that justify using validation error to select the regularization strength for the lasso provided we monitor the L1 norm of the solutions. 

\subsection{Simple Lassoed Boosting}

We first show that our result is already known to hold for a simple lassoed boosting algorithm where we do not use early stopping and where instead of trees we boost over learners of step functions. 

The class of functions $\tilde{\mathcal H}$ we consider are the indicators $\tilde h_c(x) = 1(c \le x)$ where $c \in \mathbb R^p$ and the inequality is taken to hold in all dimensions. We'll refer to these as step functions.

Presume we have a ``weak learner'' that uses the training data to produce a prediction function that is guaranteed to be in $\tilde{\mathcal H}$. We do not assume that the learner outputs the step function with lowest possible training error, the same way that CART regression trees are not globally optimal. It is instead sensible to imagine a heuristic algorithm based on recursive partitioning that should generally do a reasonable job. 

However, to simplify the task for the weak learner, we can consider searching only a \textit{finite} number of step functions $\tilde H_n \subset \tilde{\mathcal H}$. Two different functions $\tilde h, \tilde h' \in \tilde{\mathcal H}$ are indistinguishable in terms of any training loss if they have identical predictions on the observed data, i.e. $\tilde h'(\mathbf X) = \tilde h(\mathbf X)$. Therefore when finding a basis to approximately minimize the loss it suffices to consider the finite set of step functions $\tilde H_n = \{1(c \le x): c \in C_n\}$ where $C_n$ denotes the smallest lattice that includes all of the points in $(X)_i^n$ as elements. Intuitively, imagine drawing a grid over the covariate space such that our observed data are all at some of the vertices of the grid. Our set of ``knots'' $C_n$ would then be all of the vertices of the resulting grid (see Figure \ref{fig:knots}). An identical argument is used to simplify the optimization for standard regression trees: instead of considering all possible cutpoints to split a given leaf, it is sufficient to look only at the values taken in the data. We write $\tilde H_n$ with the $n$ subscript to indicate that this set is data-dependent and not fixed a-priori. We presume our weak learner produces a learned function in the set $\tilde H_n$. In general the size of $\tilde H_n$ is still massive ($n^p$) so finding a globally optimal basis is combinatorially difficult. We let $\tilde K_n = |C_n| = n^p$ denote the total number of knots in this set for a given dataset.

\begin{figure}[h]
\centering
\begin{tikzpicture}[scale=0.7]
    % Draw the grid
    \foreach \x in {1, 2, 3.5, 5} {
        \draw[gray, very thin] (\x,1) -- (\x,3.7);
    }
    \foreach \y in {1, 2, 3.7} {
        \draw[gray, very thin] (1,\y) -- (5,\y);
    }

    % Draw grey points at all vertices
    \fill[gray] (1,1) circle (2pt);
    \fill[gray] (1,2) circle (2pt);
    \fill[gray] (1,3.7) circle (2pt);
    \fill[gray] (2,1) circle (2pt);
    \fill[gray] (2,2) circle (2pt);
    \fill[gray] (2,3.7) circle (2pt);
    \fill[gray] (3.5,1) circle (2pt);
    \fill[gray] (3.5,2) circle (2pt);
    \fill[gray] (3.5,3.7) circle (2pt);
    \fill[gray] (5,1) circle (2pt);
    \fill[gray] (5,2) circle (2pt);
    \fill[gray] (5,3.7) circle (2pt);

    % Draw the points
    \fill (1,1) circle (3pt);
    \fill (2,2) circle (3pt);
    \fill (3.5,3.7) circle (3pt);
    \fill (5,2) circle (3pt);

    % Label the points
    \node [below] at (1,1) {(1,1)};
    \node [below] at (2,2) {(2,2)};
    \node [below] at (3.5,3.7) {(3.5,3.7)};
    \node [below] at (5,2) {(5,2)};

    % Draw the axes
    \draw[thick,->] (0,0) -- (6,0) node[anchor=north west] {$x_1$};
    \draw[thick,->] (0,0) -- (0,5) node[anchor=south east] {$x_2$};
\end{tikzpicture}
\caption{Black points represent the 4 observed data points $(X_1, \dots X_4)$ where $X_i \in \mathbb R^2$. Grey points represent the set of knots $C_n$.}
\label{fig:knots}
\end{figure}
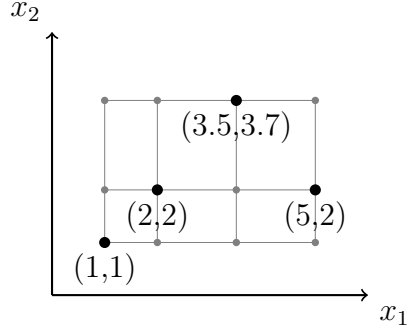

We also introduce one additional modification: if a basis has already been added to the basis matrix and the same basis would be added again in a boosting iteration, a random basis from $\tilde H_n$ is instead selected that has not already been added. In the limit of boosting iterations, we must therefore stop after adding all of the $n^p$ bases in $\tilde H_n$. Finally, after all these have been added, we solve a lasso problem with outcome vector $\mathbf Y$ and predictor matrix $\tilde H_n(\mathbf X)$:

\begin{align}
    \tilde f_n &= \argmin_{f \in \tilde{\mathcal F}_{n}(M) } P_nLf \\ 
    \tilde{\mathcal  F}_{n}(M) &= 
    \left\{
    \begin{array}{l}
        \beta_0 + \tilde H_n(x) \beta \\
        \text{s.t.}\  \|\beta\|_1 \le M
    \end{array}
    \right\}
\end{align}

This precise problem has been extensively studied \cite{hal, hal-og, Fang2019-or} and, in fact, the estimator $\tilde f_n$ converges in divergence to $f$ at a $O_P(n^{-1/3}(\log n)^{2(p-1)/3})$ rate as long as $f \in \mathcal{F}(M)$ (the set of cadlag functions of sectional variation bounded by $M$) \cite{Bibaut2019-zf}. The convergence in divergence implies convergence at the same rate in L2 norm for most common loss functions \cite{Bibaut2019-zf}. It is possible to show a slower (but still impressive) $o_P(n^{-1/4})$ rate using only generic empirical process theory \cite{hal}. The number of knots in $\tilde H_n$ may also be decreased to a specific set of $n(2^p -1)$ points \cite{hal-og} but the construction in \cite{Fang2019-or} is simpler and suffices for our purposes. The estimator $f_n$ above has been called the ``highly adaptive lasso'' (HAL) by van der Laan et al. and this is the name we will use for it as well. Astoundingly, for all of these results the convergence speed does not depend on the covariate dimension $p$ (or does, but only through a log factor), meaning that HAL effectively breaks the asymptotic curse of dimensionality. 

The key to the convergence rate result is that the lasso solution $f_n$ described above is in fact exactly a solution of the \textit{nonparametric} optimization problem \cite{Fang2019-or}:

\begin{align}
    \tilde f_n &= \argmin_{f \in\mathcal  F(M)} P_n Lf  \\ 
    \mathcal  F(M) &= 
    \left\{
    \begin{array}{l}
        f \in \mathcal F \\
        \text{s.t.} \|f\|_v \le M
    \end{array}
    \right\}
\end{align}
Although it is very expressive, the function class $\mathcal F(M)$ of cadlag functions of sectional variation bounded by $M$ is \textit{Donsker}. Generic results from empirical process theory can be used to show fast convergence rates for any empirical minimizer in a Donsker class \cite{Van_der_Vaart2000-yx, hal}.

\subsection{Rectangular Bases}

The next step is to show that we can replace the step function bases $\tilde H_n$ with a richer set of functions (indicators of rectangles). Let $ \bar h_{a,b}(x) = 1(a \le x < b)$ denote the indicator of a rectangle and let $ \bar H_n = \{ \bar h_{a,b}: a,b \in C_n\}$ be our rectangle indicators where $C_n$ is our minimal lattice from before. Denote

\begin{align}
    \bar f_n &= \argmin_{f \in\bar{\mathcal  F}_{n}(M)} P_nLf  \\ 
    \bar{\mathcal  F}_{n}(M) &= 
    \left\{
    \begin{array}{l}
        \beta_0 + \bar H_n(x) \beta \\
        \text{s.t.}\  \|\beta\|_1 \le M
    \end{array}
    \right\}
\end{align}
By adding and subtracting the step functions corresponding to the corners of a given rectangle we can always express a rectangle as a linear combination of exactly $2^p$ steps fixed at all the corners $\mathcal C(a,b)$ of the hypercube $[a,b]$. Specifically, $ \bar h_{a,b} = \sum^{2_p}_{c \in \mathcal C(a,b)} (-1)^{\sum_j 1(c_j=a_j)} \tilde h_c$. 

In a slight abuse of notation, if we take $\tilde H_n$ to be the row vector of all $\tilde K_n$ step functions and we take $\bar H_n$ to be the row vector of all $\bar K_n$ rectangular bases, we can write $\bar H_n = \tilde H_n A$ where $A \in \mathbb R^{\tilde K_n \times \bar K_n}$. By virtue of the above we know each column of $A$ has $2^p$ nonzero entries and that these values are all $\pm 1$. 

What we'd like to show is that $\bar f_n$ (lasso over rectangles) retains the rate of $\tilde f_n$ (lasso over steps). The rate attained by the latter is due to $\tilde {\mathcal F}_n(M) \in \mathcal F(M)$ being a Donsker class having well controlled entropy \cite{Bibaut2019-zf} and so we would like to show that $\bar{\mathcal F}_n(M) \in \mathcal G$ for some class $\mathcal G$ that enjoys similar properties.

To do this, note that $\bar H_n\beta = (\tilde H_n A)\beta$. The coefficients of this function as expressed over the steps are $A\beta = \sum^{\bar K_n}_k A_k \beta_k$ where $A_k$ is the $k$th column of $A$. So $\|A\beta\|_1 \le \sum^{\bar K_n}_k \|A_k\|_1 |\beta_k| = \sum^{\bar K_n}_k 2^p |\beta_k| = 2^p\|\beta\|_1$. Thus 
$\tilde{\mathcal F}_n(M) \subset \bar{\mathcal F}_n(M) \subset \tilde{\mathcal F}_n( M 2^{p} ) \subset \mathcal F(M2^{p} )$. This establishes $\mathcal F(M2^{p})$ as the parent Donsker class than $\bar f_n$ lives in.

Let $f$ be our true loss minimizer in $\mathcal F(M2^{p} )$ and presume we know that actually $f \in \mathcal F(M)$. The arguments used in \cite{Bibaut2019-zf} show that $\bar f_n$ converges in divergence in $\mathcal F(M2^{p} )$ at the same rate as $\tilde f_n$ converges in $\mathcal F(M)$. The only difference is the value of $M$, which enters the rate as a constant factor.

Since we are allowed rectangular bases we can now use a tree boosting algorithm to add them in sequentially. Each tree that is added to the ensemble contributes some number of rectangular bases corresponding to its leaves. If we force this boosting algorithm to continue adding trees until every rectangle in $H_n$ has been included and then finally perform a lasso regression over the leaves, then the result is the above estimator $\bar f_n$ which we have proven is fast-converging. In order to find all rectangles in $H_n$ the tree depth $D$ must be set to $p$ at a minimum.

By identical arguments we may now progress to $p$-depth regression trees $h(x) = \sum^{2^p} v \bar h(x)$ as long as the leaf values $v$ of all trees are bounded by some constant $V$. Fast convergence is now guaranteed with respect to $\mathcal F(VM(2^p)^2)$. The only issue is that the set of all regression trees constructed out of the rectangles in $\bar H_n$ is not finite due to the infinite number of possible leaf values. Therefore to construct a finite basis with the desired properties we simply augment $\bar H_n$ with a large, finite number of bounded regression trees constructed from $\bar H_n$. It does not matter which (they may be chosen data-adaptively) or exactly how many. Call the resulting basis $H_n$.

\subsection{Early-stopping Validation}
\label{sec:early-stopping}

We have shown that lassoed boosting converges quickly if we let it exhaust all possible bases because the result at the end is effectively the HAL estimator. 

Obviously we would like to be able to stop the boosting algorithm well before we have included all of the possible bases and we would like to fit trees of depth much less than $p$. We propose to do this with early-stopping validation. Generally, the point of early-stopping validation is to fit a sequence of nested models and to stop once error on a validation set begins to increase. This strategy is commonly employed as a model selection heuristic that prevents having to fit larger and larger models ad infinitum \cite{Luo2016-cb, Zhang2005-xh}. For example, several popular packages for gradient boosting allow the user to fit models without specifying the number of trees a-priori: given a validation dataset, the software continues adding trees until validation loss begins to increase. More generally this can be done in an ad-hoc fashion: a user might manually expand the upper/lower bound value of a hyperparameter during grid search model selection if it looks like a minimum has not been reached. Often these hyperparameters directly correspond to moving upwards in a series of nested statistical models (e.g. increasing number or depth of trees in tree ensembles, loosening regularization in elastic net).

Our concern here is to prove that early-stopping validation does not sacrifice the convergence rate we would obtain by letting our lassoed boosting algorithm run for the full number of possible iterations at depth $D=p$. The key is to cast the lassoed boosting algorithm as a sequence of empirical loss minimizations in larger and larger function classes (adding bases, increasing depth). 

Let $h_{k,n}$ be the chosen basis function in the $k$th boosting iteration and let $H_{k,n} = [h_{1,n} \dots h_{k,n}]$ be the set of bases after $k$ iterations. Our estimator at that point in the algorithm is the lasso regression of $\mathbf{Y}$ onto $H_{k,n}(\mathbf X)$:

\begin{align}
    f_{k,n} &= \argmin_{f \in\mathcal  F_{k,n}(M)} P_nLf \\ 
    \mathcal  F_{k,n}(M) &= 
    \left\{
    \begin{array}{l}
        \beta_0 + H_{k,n}(x) \beta \\
        \text{s.t.}\  \|\beta\|_1 \le M
    \end{array}
    \right\}
\end{align}

Note that the class $\mathcal F_{k,n}(M)$ itself depends on the training data, but nonetheless as we add bases we know $\mathcal F_{k,n}(M) \subset \mathcal F_{k+1,n}(M)$. Our estimator $f_{k,n}$ at each iteration is the empirical loss minimizer in the corresponding class $\mathcal F_{k,n}$.

Our main analytical result is that early-stopping validation of an estimator that minimizes empirical loss in a nested sequence of models embedded in a Donsker class preserves the convergence rate of the empirical minimizer in the last model relative to the true minimizer in the parent Donsker class. The proof is given in the appendix and relies only on generic results from empirical process theory. 

\begin{restatable}[Early-stopping validation is rate-preserving]{theorem}{THMstopping}
Consider a nested, possibly random or data-dependent, sequence of function classes 
$$ (\mathcal F_1 \subset \mathcal F_2 \subset \dots \mathcal F_k \subset \dots \mathcal F_{K_n} )_n \subset \mathcal F
$$ 
all contained in some Donsker class $\mathcal F$ and all containing the set of constant functions. Let
\begin{itemize}
    \item $f_{k,n} = \argmin_{f \in \mathcal F_{k,n}} P_n L f$ be an empirical minimizer of $L$ in each model
    \item $f = \argmin_{\phi \in \mathcal F} P L \phi$ be the true, unknown minimizer in the parent Donsker class. 
\end{itemize}

Define the difference in validation-set error from one iteration to the next to be $\Delta_{k,n} = \tilde P_n Lf_{k+1,n} - \tilde P_n Lf_{k,n}$ and choose $k^*_n$ to be the smallest value $k$ such that $\Delta_{k,n} > \epsilon_n$ for some positive $\epsilon_n$ that is allowed to shrink slower than $O(n^{-1/2})$. 

If $P(Lf_{K_n,n} - Lf) = o_P(r_n)$ for some $r_n \to 0$ then $P(Lf_{k^*_n,n} - Lf) = o_P(r_n)$.
\label{thm:early-stopping}
\end{restatable}

The way the stopping point $k^*$ is chosen means that the analyst never needs to fit past $f_{k^*+1,n}$. In contrast, consider the standard (cross-)validation scheme where $k^*$ is chosen as the global minimizer of validation loss $\tilde P_n L f_{k,n}$ over all $k$ in some finite set. Existing results show that this scheme preserves the convergence rate of the fastest estimator in the ensemble, even if it grows at slower than a given rate in $n$ \cite{Dudoit2005-fi}. The downside is that computing the global minimum means fitting $f_{k,n}$ for all $k$.
With early-stopping validation, we are instead satisfied with the first \textit{local} minimum we find as we increase $k$. This means we don't have to fit any models past that point, saving computation and allowing us to select among massive numbers of models. This has been heuristically justified in the past by arguing that larger models will continue to overfit more and more, so we are often right in guessing that the first local minimum is also the global minimum. 
That is not always true in finite samples, but our result establishes that is true in an asymptotic sense for nested sequences of empirical minimizers in Donsker classes. 

Statistically, the intuition for this result is that the difference in validation error from one step to the next $\Delta_{k,n}$ is an empirical mean with zero or negative expectation until reaching the true minimizer in $k$. That means that all we have to do to avoid stopping too early is have enough validation samples so that any positive random variation is smaller than $\epsilon_n$ with probability tending to one. On the other hand as we increase the number of training samples we expect the validation error curve to flatten out (i.e. it becomes harder and harder to overfit). Since $\epsilon_n$ is a positive number, the chance of stopping at any $k$ also goes to zero with enough validation data since the random variation will become unlikely to \textit{exceed} $\epsilon_n$. Figure \ref{fig:val-loss} shows a conceptual illustration.

\begin{figure}[ht]
    \centering
    \includegraphics[width=4in]{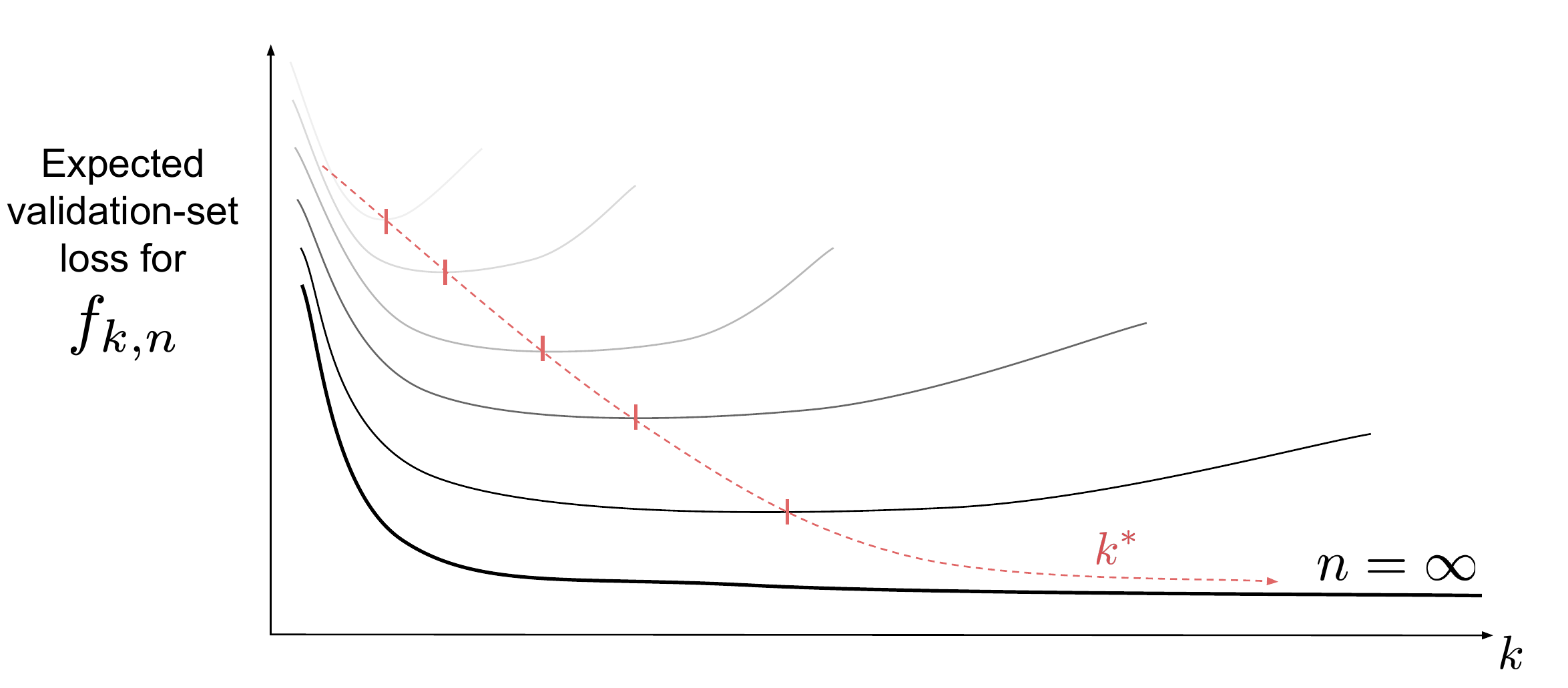}
    \caption{As the sample size $n$ increases, it becomes harder to overfit and the expected validation curve flattens out. With increasing $n$ the variation around this mean also decreases at a root-$n$ rate (not depicted).}
    \label{fig:val-loss}
\end{figure}

% The trade-off is that standard (cross-)validation can handle completely arbitrary estimators, whereas with early stopping the estimators must be empirical loss minimizers of a nested series of models embedded in a Donsker class. This is mitigated by the fact that we are still perfectly justified in using an outer, standard validation setup to select among a finite ensemble of arbitrary estimators, one of which happens to be the estimator selected via an inner, early-stopping validation procedure of the kind we describe here.

Theorem \ref{thm:early-stopping} easily extends to a early-stopping $V$-fold \textit{cross-}validation scheme where $k^*$ is chosen based on the first local minimum of the validation error averaged over the folds $\frac{1}{V}\sum_v \tilde P_n^{(v)} L f_{k,n}^{(-v)}$. In this setting, the final estimator will typically be re-fit for $k^*$ iterations using the full data and in general $\mathcal F_{k^*,n}^{(v_1)} \ne \mathcal F_{k^*,n}^{(v_2)} \ne \mathcal F_{k^*,n}$ for any two validation sets $v_1, v_2$. In other words, the exact sequence of bases in each validation sample will be different. This may seem unsettling, but the fact that our result uniformly covers cases where the sequence of statistical models is chosen adaptively or randomly means that we can select the number of iterations $k^*$ without regard to the precise models and still obtain our conclusion.

The uniformity over all data-adaptive nested sequences of models is what justifies many of our practical choices for lassoed boosting. There is effectively no restriction on the order in which we must add in the bases. This justifies our choice to first generate a large number of candidate bases using vanilla boosting with early stopping over both tree depth and number of trees and only then beginning to regularize all the leaf values with the lasso. This saves a huge amount of computation that would otherwise be spent solving larger and larger lasso problems. Practically we also never have to increase the tree depth past what was selected by the boosting early stopping: as long as we promise to add these deeper trees after all of the more shallow trees have been expended the theorem holds. In practice, validation error will begin to increase far before this happens.

It is also important to point out that $\epsilon_n$ may be taken as a fixed constant for all $n$ (e.g. this can be baked into the algorithm as a default setting). But having $\epsilon_n$ be allowed to shrink with the sample size is extremely beneficial: if we were to use a fixed cutoff $\epsilon_n = \epsilon$ and apply our early-stopping algorithm to larger and larger samples, we would expect the discovered stopping point $k^*$ to increase because larger samples combat variance (overfitting). However, we don't want $k^*$ to get too large in practice because that increases computational intensity. By allowing $\epsilon_n$ to shrink with sample size, we can be more and more sensitive to increases in validation loss as $n$ increases. That makes us more aggressive about stopping early in large samples, thus partially slowing down the rate at which $k^*$ grows while still preserving the asymptotics. We expect users may choose $\epsilon$ to facilitate computation, meaning that in different problems they may use different values to keep runtimes reasonable. In other words, users will naturally choose shrinking $\epsilon_n$, even if they do not realize it. Allowing $\epsilon_n$ to shrink (but not too fast) helps protect the statistical properties from this aggressive user behavior.

Our theory also trivially covers any early stopping schemes that are less aggressive than checking for the first time that $\Delta_{k,n}>\epsilon_n$. For example, to try to get over local minima and improve empirical finite-sample performance we might stop only after $\Delta_{k,n}>\epsilon_n$ some number of iterations in a row. Schemes of this kind are common in machine learning applications that use early stopping.

\subsection{Lasso Tuning}
Until now we have presumed that all lasso problems are solved using a fixed, known L1 bound $M$. In practice this is not possible because all common lasso implementations solve the Lagrangian lasso problem where the L1 bound is reformulated as a penalty $\lambda$. The relationship between $M$ and $\lambda$ that share a solution is monotonic, but otherwise depends on the dataset. Since we are changing the bases from iteration to iteration, it does not suffice to set $\lambda$ at a fixed value. Moreover this would be bad for performance.

Our first concern is to ensure that the lasso bound does not decrease from iteration to iteration. Our early-stopping theory only holds over a sequence of increasing models. If the lasso bound shrinks while the number of bases increases then the relationship between the function classes searched over from one iteration to the next is indeterminate. 

One approach is to compute a full lasso path at each iteration and take the best solution (in validation error) that has L1 norm larger than the L1 norm of the solution from the previous iteration (but smaller than some large predefined upper bound). In practice this begs the question of how to set the first L1 norm bound. If set by validation error, we found in initial testing that the L1 norm increases too quickly, requiring smaller and smaller $\lambda$ and leading to convergence issues. It is also counterintuitive: as we add bases, we would like to \textit{increase} rather than \textit{decrease} regularization to counteract overfitting.

Our approach instead exploits the uniformity of our early stopping result. At each iteration we find $\lambda$ that gives minimum validation error, subject again to a fixed global upper bound on L1 norm. We then look back over all lasso solutions (full regularization paths, not just optima) that have a smaller L1 norm from all previous iterations. If any of these attained a smaller validation error we stop the algorithm. This is justified because we have found an increase in validation error over \textit{some} increasing sequence of models. By fitting full regularization paths at each iteration we can effectively search over multiple nested sequences in a partial order over a large set of models, prioritizing small validation error.

\section{Demonstration}
\label{sec:demo}

Our goal in this section is to empirically demonstrate that lassoed tree boosting (LTB) has a fast convergence rate, performs well in real-data settings, and runs quickly enough for practical use. 

In all experiments we use simple training-validation-test splits, using the same, fixed validation data to do all tuning and simply returning the fit from the training data without refitting with the validation data before computing test-set error. This is to save computation since one of our benchmarks (HAL) is very slow. All datasets are split into training-validation (90\%) and test sets (10\%) and then the training-validation data are randomly split into training (80\%) and validation (20\%).

We use the xgboost implementation of GBT with learning rate set to $0.05$ and other parameters left to defaults. For a fixed depth, we tune the number of trees with early stopping validation, requiring 3 increases in validation error to stop. We then tune depth via early stopping in an outer loop, starting with depth 1 and increasing by 1 in each iteration, stopping at the first increase in error. 

Our LTB implementation starts with the tuned GBT fit described above. We then compute the lasso path over all tree predictions using the celer lasso solver with default parameters. We add 10 trees to the boosting model and then repeat the process, stopping according to the scheme described above.

In our experiments we compare to tuned GBT, lasso on the original covariates (optimizing $\lambda$ with validation error), and a pure HAL implementation in the spirit of van der Laan \cite{hal}. For fair comparison we use the same lasso solver for HAL and LTB.

\subsection{Rates in Simulation}
A simple simulation provides some evidence that our theoretical rate holds up in practice. We simulated 100 repetitions of datasets of increasing sizes ($n$) from the following data-generating process:

\begin{align*}
X_1 &\sim \text{Unif}(-4,4)\\
X_2 &\sim \text{Bern}(0.5)\\
Y &= -0.5X_1+\frac{X_2X_1^2}{2.75}+X_2 + \mathcal N\left(0,1 \right)
\end{align*}

We applied each estimator (HAL, LTB, and GBT in this case) to each simulated dataset to obtain estimates $f_n$ for each. We then calculated an approximate L2 error $\tilde P_n (f-f_n)^2$ using a large held-out dataset of 100,000 observations from the same data-generating process. The results are shown in Figure \ref{fig:convergence}

\begin{figure}[ht]
    \centering
    \includegraphics[width=5in]{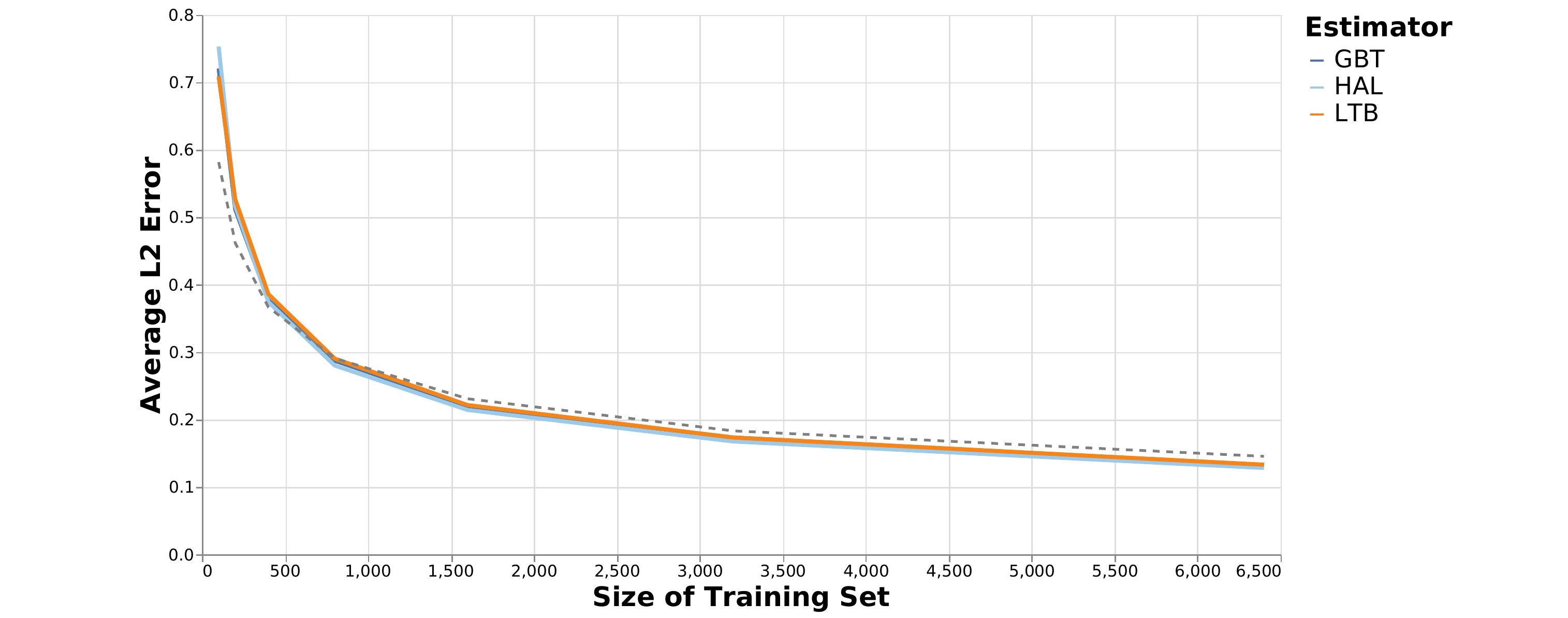}
    \caption{Empirical convergence rates of GBT, HAL, and LTB in both data-generating processes. The theoretical $n^{-1/3}$ rate is also shown.}
    \label{fig:convergence}
\end{figure}

The results confirm that LTB attains the same empirical rate as HAL, up to a constant factor. GBT by itself also appears to attain this rate. All three are approximately equal to the theoretical rate up to constant factors.

\subsection{Performance and Timing on Real Data}

Here we demonstrate how LTB performs relative to GBT, HAL, and lasso on 12 regression datasets from the UCI Machine Learning Repository \cite{Dua2017-yp}. It is not computationally feasible to run HAL on most of these datasets--we omit results for HAL on datasets where HAL had not fit after more than an hour of running. All results are averages across 10 runs except for the slice and yearmsd datasets which we only ran once.

Our results are shown in table \ref{tbl:results}. LTB and GBT perform almost identically and in all cases are much better than lasso and on par with HAL where the comparison is possible. 
LTB incurs some computational cost relative to GBT but pales in comparison to the computational cost of HAL for the larger datasets.
Interesting, the added compute time for LTB does not always scale with the size of the data. This is because in larger datasets the GBT early stopping algorithm tends to fit deeper trees (and therefore fewer of them) and this reduces the dimension of the lasso problems. 

\begin{table}
\centering
\begin{tabular}{lrr|llll}
\toprule
    data &   p &      n &            LTB &            GBT &           HAL &         LASSO \\
\midrule
   yacht &   6 &    308 &   0.90 (10.69) &    0.90 (4.68) &   0.72 (0.92) &   8.92 (0.01) \\
  energy &   8 &    768 &   0.40 (30.82) &   0.40 (21.46) &  0.43 (45.80) &   4.14 (0.01) \\
  boston &  13 &    506 &    3.35 (5.92) &    3.43 (4.47) & 3.66 (916.61) &   5.02 (0.01) \\
concrete &   8 &   1030 &   4.70 (43.29) &   4.87 (37.15) & 4.02 (134.01) &  10.40 (0.01) \\
    wine &  11 &   1599 &    0.64 (6.01) &    0.63 (3.58) &           -- &   0.67 (0.03) \\
   power &   4 &   9568 &   3.41 (56.92) &   3.46 (28.88) &           -- &   4.59 (0.04) \\
  kin8nm &   8 &   8192 &   0.12 (96.50) &   0.12 (60.40) &           -- &   0.21 (0.03) \\
   naval &  17 &  11934 &  0.00 (107.98) &   0.00 (56.19) &           -- &  0.01 (10.51) \\
 protein &   9 &  45730 &  1.94 (611.38) &   1.94 (96.80) &           -- &   2.50 (0.33) \\
    blog & 280 &  52397 & 23.49 (185.49) &   23.46 (9.90) &           -- & 28.25 (13.51) \\
   slice & 384 &  53500 & 1.23 (3350.61) & 1.24 (3067.95) &           -- & 8.33 (121.71) \\
 yearmsd &  90 & 515345 & 8.54 (4616.82) & 8.54 (1543.05) &           -- &  9.49 (40.09) \\
\bottomrule
\end{tabular}
\caption{Performance on real datasets, ordered by $p \times n$. Entries are mean RMSE and timing (in seconds) across repetitions for all experiments.}
\label{tbl:results}
\end{table}

\section{Discussion}
\label{sec:discussion}

We have demonstrated a regression algorithm, lassoed tree boosting, which is computationally scalable while maintaining a $O_P(n^{-1/3}(\log n)^{2(p-1)/3})$ L2 convergence rate in a nonparametric setting. The method is easy to implement with off-the-shelf software and runs as quickly as solving a moderately large lasso problem and fitting a boosting model. We showed that our algorithm retains the performance of standard gradient boosting on a diverse group of real-world datasets. Along the way, we proved a general result pertaining to empirical loss minimization in nested Donsker classes. This may be of independent interest for the development of future statistically-justified machine learning methods. 

Our empirical results also show that GBT's early-stopping performance is extremely similar to LTB's without any special modification.
This begs the question of why we should pay the computational cost of the lasso step in practice. 
We have no reply other than to say that LTB has a rate guarantee and GBT does not. 
Therefore in cases where such a guarantee is required (e.g. regressions for efficient estimators in causal inference settings) LTB provides peace of mind for a relatively minor cost.

However, the results suggest that boosting by itself may converge quickly in the space of bounded sectional variation cadlag functions. 
Our theoretical results may be a step towards showing this rigorously. 
Existing results show that (in special cases) the added ``complexity'' of a single boosting iteration is exponential \cite{buhlmann2003boosting}: it may be that standard boosting implicitly restricts the variation norm in some way even without the lasso step. There may also be connections to recent literature that shows how discrete gradient descent creates implicit regularization in overparametrized models \cite{barrett2020implicit, gidel2019implicit}. 
Having a general result that explicitly covers standard implementations of gradient boosting would obviate the need for the additional complexity introduced by lassoed boosting.

Our results also shed some light on why boosting is so successful by itself. The standard approach of regularizing via a small learning rate $\eta$ mimics L1 shrinkage as $\eta \to
0$, assuming a globally optimal tree can be fit in each iteration \cite{Efron2004-jv, Zhao_undated-fg, Luo2016-cb}. Our early-stopping result makes it clear that, at least asymptotically, trees can be added in a suboptimal order without compromising the rate. 

Our early stopping result is extremely general and in theory covers any data-adaptive method of basis selection. For example, using the HAL bases suggested by van der Laan \cite{hal} we might first include all the "main terms" bases $1(c_j \le x_j)$, then add in the two-way interactions $1(c_j \le x_j, c_l \le x_l)$, etc. and stop early based on validation error. This is justified, but not necessarily practical because the number of bases still increases too quickly. The intuitive reason that boosting with trees works so well is that it already approximates the lasso path over the set of all regression trees and thus provides the ``right'' ordering of trees, i.e. it prioritizes trees that should end up having nonzero coefficients in the full model.

Although lassoed boosting is inspired by HAL, the two algorithms are not fungible. Recent work has shown how HAL can be used as a plugin estimator for pathwise differentiable estimands \cite{Van_der_Laan2019-bt} and that HAL is pointwise asymptotically normal and uniformly consistent as a regression estimator \cite{vanderlaan2023higher}. These results are due to the fact that HAL asymptotically solves a large huge number of score equations (one per nonzero basis in the lasso). In theory, lassoed boosting does this as well, but part of the way lassoed boosting eases the computational burden is precisely by reducing the number of bases. Thus HAL retains important theoretical benefits. Investigating how to export these benefits while improving the computational burden or whether it is possible will be the subject of future work.

\subsection*{Acknowledgments}

{\small
The authors are eternally grateful to Tianyue Zhou for notational corrections throughout the manuscript and to several anonymous referees for important structural improvements.
}

\bibliographystyle{plain}
\bibliography{references}

\appendix

\section{Lemmas}

First we remind ourselves of an elementary asymptotic result:
\begin{lemma}
Let $U_n \in \{0,1\}$ and assume $P\{U_n=1\} \to 0$. Then $a_n U_n \overset{P}{\to} 0$ for any $a_n$.
\label{thm:prob-conv}
\end{lemma}

\begin{proof}
By definition of convergence in probability we must show that $P\{a_nU_n > \epsilon\} \to 0$ for any fixed $\epsilon > 0$. But this is clear because 
\begin{align*}
P\{a_nU_n > \epsilon\} 
&= 
\underbrace{P\{a_nU_n > \epsilon | U_n=0\}}_{0} P\{U_n = 0\} + 
P\{a_nU_n > \epsilon | U_n=1\} P\{U_n = 1\} \\
&\le P\{U_n = 1\} \to 0
\end{align*}
\end{proof}

We also sketch the proofs for some generic results from empirical process theory that we will use several times. We will use the following definitions for subsequent results:

\begin{itemize}
    \item $\{\mathcal F_k\} \subseteq \mathcal F$: a sequence of Donsker classes all of which are contained in a parent Donsker class $\mathcal F$.
    \item $\mathcal M = \{Lf:f \in \mathcal F\}$: also a Donsker class.
    \item $f_{k} = \argmin_{f \in \mathcal{F}_{k}} PLf$: a sequence of loss minimizers in $\mathcal F_{k}$ (define $f$ equivalently w.r.t. $\mathcal F$)
    \item $f_{k, n} = \argmin_{f \in \mathcal{F}_{k}} P_nLf$: a sequence of empirical loss minimizers in $\mathcal F_{k}$ (define $f_n$ equivalently w.r.t. $\mathcal F$)
\end{itemize}

\begin{lemma}
If $\mathcal M$ is Donsker, $(P_n-P)m_n = O_P(n^{-1/2})$ for any (possibly random) sequence $m_n \in \mathcal M$. 
\label{thm:unif-cnv}
\end{lemma}
\begin{proof}
By definition, since $\mathcal M$ is Donsker, $\sqrt{n}(P_n-P) \rightsquigarrow G$ where $G$ is a tight Gaussian process indexed by $\mathcal M$ with sample paths in $l^\infty(\mathcal M)$ (bounded functions from $\mathcal M$ to $\mathbb R$, equipped with supremum norm). Tightness means that with high probability the sample paths of $G$ are in a compact set and can thus be covered with a \textit{finite} number of balls. The center of each of these balls is a bounded function since all sample paths are bounded. Therefore all functions in this high-probability set are also bounded, and moreover are bounded uniformly since the number of balls is finite. 

That establishes the fact that $\sup_{m\in \mathcal M} \sqrt{n}(P_n-P)m = O_P(1)$, since for any level of probability $\delta$ we can eventually find a high-probability set where all the contained functions are bounded. From there we get that $(P_n-P)m_n \le \sup_{m\in \mathcal M} (P_n-P)l = O_P(n^{-1/2})$ as desired. 

\textbf{Note}: as a trivial corollary we also have that $(\tilde P_n-P)m_n \le \sup_{m\in \mathcal M} (\tilde P_n-P)m = O_P(n^{-1/2})$ 
\end{proof}

\begin{corollary}
$P(Lf_{k_n,n}-Lf_{k_n}) = O_P(n^{-1/2})$ for any sequence $k_n$.
\label{thm:unif-conv-min}
\end{corollary}
\begin{proof}
\begin{align*}
    0 \le P(Lf_{k_n,n}-Lf_{k_n})
    &= 
    -(P_n-P)(Lf_{k_n,n}-Lf_{k_n}) + P_n(Lf_{k_n,n}-Lf_{k_n})
    \\
    &\le -(P_n-P)(Lf_{k_n,n}-Lf_{k_n}) 
\end{align*}
which is $O_P(n^{-1/2})$ by Lemma \ref{thm:unif-cnv}.
\end{proof}

\begin{corollary}
$P(Lf_{k_n, n}-Lf_{k'_n, n}) \le O_P(n^{-1/2})$ for any sequences $k'_n < k_n$.
\label{thm:nested-cnv}
\end{corollary}
\begin{proof}
\begin{align*}
    P(Lf_{k_n, n} - Lf_{k'_n, n}) 
    &= P(Lf_{k_n, n} - Lf_{k_n}) 
    + P(Lf_{k_n} - Lf_{k'_n})
    + P(Lf_{k'_n} - Lf_{k'_n, n}) \\
    &\le 
    |P(Lf_{k_n,n} - Lf_{k_n})|
    + |P(Lf_{k'_n, n}-Lf_{k'_n})| \\
    &=  P(Lf_{k_n, n} - Lf_{k_n})
    + P(Lf_{k'_n, n}-Lf_{k'_n})
\end{align*}
and Corollary \ref{thm:unif-conv-min} and ensures these terms are $O_P(n^{-1/2})$. 
\end{proof}

\begin{corollary}
$P_nLf_{k,n} = O_P(1)$
\label{thm:bounded-loss}
\end{corollary}
\begin{proof}
$$P_nLf_{k,n} = (P_n -P)Lf_{k,n} + P(Lf_{k,n}-Lf_k) + PLf_k$$
The first and second terms are $O_P(n^{-1/2})$ by Lemma \ref{thm:unif-cnv} and Corollary \ref{thm:unif-conv-min} and the last term is a finite, fixed number.
\end{proof}

\section{Proof of Theorem \ref{thm:early-stopping}}

\THMstopping*

\begin{proof}

We show that the early-stopping solution asymptotically converges to the empirical minimizer in the final class as quickly as that minimizer converges to the truth (e.g. our early stopping solution converges as quickly to the HAL solution as the HAL solution converges to the truth). To do this we show that in large-enough samples, the validation error starts to trend monotonically downward in $k$ with high probability. Therefore early stopping does not kick in and with high probability we obtain the empirical minimizer in the final class.

\paragraph*{}
We begin by decomposing

\begin{align}
    P(Lf_{k^*_n,n} - Lf) = 
      P(Lf_{k^*_n,n} - Lf_n) 
    + \underbrace{P(Lf_{n} - Lf)}_{o_P(r_n)\ \text{by assumption}}
\end{align}

The proof is complete if the first term on the right-hand side is also $o_P(r_n)$; that is, if we can show that the early-stopping solution $f_{k^*_n,n}$ converges in expected loss to the final empirical minimizer $f_n$ at a rate of $o_P(r_n)$ or faster. To accomplish that we decompose

\begin{align}
    |P(Lf_{k^*_n,n} - Lf_n)|
    &= \left|\sum_k^{K_n} 1(k^*_n=k) P(Lf_{k,n} - Lf_n)\right| \\
    &\le 1(k^*_n = K_n) \cancel{|P(Lf_{K_n,n} - Lf_n)|}
      + 1(k^*_n < K_n) \max_{k < K_n} |P(Lf_{k,n} - Lf_n)| \label{eq:stopping-decomp}
\end{align}

The cancellation comes by our definition that $f_n = f_{K_n,n}$. Using that $1(k^*_n = K_n) = 1 - 1(k^*_n < K_n)$ we therefore establish the following sufficient conditions to ensure that the left-hand side above is $o_P(r_n)$:

\begin{enumerate}
    \item $\max_{k < K_n} |P(Lf_{k,n} - Lf_n)| = O_P(1)$
    \item $1(k^* < K_n) = o_P(r_n)$
\end{enumerate}

\paragraph*{}
{\bf Condition (1)}
Let
$$\bar k_n = \argmax_{k < K_n} |P(Lf_{k,n} - Lf_n)|$$
identify the model where the generalization error of the empirical minimizer is furthest to that of the empirical minimizer in the final model. Condition 1 reduces to showing $|P(Lf_{\bar k_n,n} - Lf_n)| = O_P(1)$ which is easily done:

\begin{align}
    |P(Lf_{\bar k_n,n} - Lf_n)| 
    &\leq 
    |(P_n-P)(Lf_{\bar k_n,n} - Lf_n)|
    + |P_n(Lf_{\bar k_n,n} - Lf_{n})| \\
    &\leq 
    |(P_n-P)(Lf_{\bar k_n,n} - Lf_n)| + |P_n Lf_{1,n}|
\end{align}

The first term is $O_P(n^{-1/2})$ by Lemma \ref{thm:unif-cnv} and the second is $O_P(1)$ by Corollary \ref{thm:bounded-loss}.

\paragraph*{}
{\bf Condition (2)}
Now we can proceed to show condition (2), which was that $1(k^* < K_n) = o_P(r_n)$. This follows from Lemma \ref{thm:prob-conv} as long as $P\{k^*<K_n\} \to 0$, i.e. the probability of early stopping goes to zero as the sample size increases. By definition this can only happen if the probability of validation error increasing from one iteration to the next goes to zero at every iteration. Formally, we need $P\{\max_{k \le K_n} \Delta_{k,n} > \epsilon_n\} \to 0$.

We will show $P\{ \Delta_{k_n,n} > \epsilon_n\} \to 0$ for any sequence $k_n$ which thus implies what we want as one example. Expanding from the definition of $\Delta_{k_n,n}$, we have
\begin{align*}
\Delta_{k_n,n} 
    &= \tilde P_n(Lf_{k_n+1,n} - Lf_{k_n,n}) \\
    &= (\tilde P_n-P)(Lf_{k_n+1,n} - Lf_{k_n,n}) + 
    P(Lf_{k_n+1,n} - Lf_{k_n,n})
\end{align*}
The first term is $O_P(n^{-1/2})$ by Lemma \ref{thm:unif-cnv} and the second term is bounded above by $O_P(n^{-1/2})$ by Corollary \ref{thm:nested-cnv}. Therefore we have that $\Delta_{k_n,n} \le O_P(n^{-1/2})$. Thus we can afford $\epsilon_n$ shrinking at any rate slower than $n^{-1/2}$ and still get $P\{ \Delta_{k_n,n} > \epsilon_n\} \to 0$ for any sequence $k_n$. This is satisfied by assumption.

\end{proof}

\end{document}